\documentclass[twoside,11pt]{article}
\usepackage{pgm}

\usepackage[utf8]{inputenc}
\inputencoding{utf8}
\usepackage{times}

\usepackage{soul}
\usepackage{url}
\usepackage[hidelinks]{hyperref}
\usepackage{graphicx}
\usepackage{amsmath}
\usepackage{booktabs}
\usepackage{algorithm}
\usepackage[noend]{algorithmic}
\usepackage{stackrel}
\usepackage[english]{babel}
\usepackage{amsfonts}
\usepackage{amssymb}
\usepackage{multirow}
\usepackage{pdflscape}
\usepackage{mathtools}
\usepackage{url}
\usepackage{graphicx}
\usepackage{bm}
\usepackage{colortbl}
\usepackage{pgf,tikz}
\usetikzlibrary{arrows,automata}
\usepackage{hyperref,color,soul,booktabs}
\setulcolor{blue}

\newcommand{\graph}{\mathbf{G}}
\newcommand{\score}[2]{\sigma_{#1}({#2})}
\newcommand{\vertex}[1]{V_{#1}}
\newcommand{\opt}{\mathit{OPT}}

\newcommand{\vertices}{\mathbb{V}}
\newcommand{\edges}{\mathbb{E}}

\newcommand{\parents}{\Pi}

\newcommand{\problem}{\mathit{\epsilon}\text{BNSL}}

\newcommand{\pen}[2]{t_{#1}({#2})}

\newcommand{\BibTeX}{\textsc{B\kern-0.1emi\kern-0.017emb}\kern-0.15em\TeX}

\ShortHeadings{A Score-and-Search Approach to Learning Bayesian Networks with Noisy-OR Relations}{Sharma et al.}

\begin{document}

\title{A Score-and-Search Approach to Learning Bayesian Networks with Noisy-OR Relations}


\author{\Name {Charupriya Sharma} \Email {c9sharma@uwaterloo.ca}
\and
\Name {Zhenyu A. Liao} \Email {z6liao@uwaterloo.ca}\\
\addr David R. Cheriton School of Computer Science, University of Waterloo, Canada
\and
\Name {James Cussens} \Email {james.cussens@bristol.ac.uk }\\
\addr Department of Computer Science, University of Bristol, UK
\and
\Name {Peter van Beek} \Email {vanbeek@uwaterloo.ca}\\
\addr David R. Cheriton School of Computer Science, 
University of Waterloo, Canada
}

\maketitle
\begin{abstract}
A Bayesian network is a probabilistic graphical model that consists of a directed acyclic graph (DAG), where each node is a random variable and attached to each node is a conditional probability distribution (CPD). A Bayesian network can be learned from data using the well-known score-and-search approach, and within this approach a key consideration is how to simultaneously learn the global structure in the form of the underlying DAG and the local structure in the CPDs. Several useful forms of local structure have been identified in the literature but thus far the score-and-search approach has only been extended to handle local structure in form of context-specific independence. In this paper, we show how to extend the score-and-search approach to the important and widely useful case of noisy-OR relations. We provide an effective gradient descent algorithm to score a candidate noisy-OR using the widely used BIC score and we provide pruning rules that allow the search to successfully scale to medium sized networks. Our empirical results provide evidence for the success of our approach to learning Bayesian networks that incorporate noisy-OR relations.
\end{abstract}
\begin{keywords}
Bayesian networks; structure learning; causal noisy-OR.
\end{keywords}
\section{Introduction}

Bayesian networks (BNs) are widely used probabilistic graphical models with applications in knowledge discovery, decision support, and prediction~\citep{Darwiche09,KollerF09}.
A BN can be learned from data using the well-known \emph{score-and-search} approach, where a scoring function is used to evaluate the fit of a proposed BN to the data in the space of directed acyclic graphs (DAGs). Current implementations of this approach such as~\citep{YuanMW11}, \citep{BartlettC13}, and~\citep{vanBeek2015} consider only conditional probability tables (CPTs) as representations for the underlying conditional probability distributions (CPDs) for discrete variables. However, the size of the CPT for a variable grows exponentially as the number of parents increases. For example, the CPT of a binary child node with $n$ binary parents requires $2^{n+1}$ probabilities. This presents a practical difficulty in parameter estimation and inference and has motivated many structured representations for CPDs that exploit the relationship between a child and its parents and aim at reducing model complexity.

A widely used local structure is the noisy-OR relation~\citep{good1961,Pearl1988} and its generalizations such as leaky noisy-OR~\citep{Henrion87} and noisy-MAX~\citep{diez93}. These relations model the CPD over causes (parents) and effects (children). The noisy-OR assumes a form of causal independence (CI) and allows one to specify a CPT with just $n$ parameters instead of $2^{n+1}$. \cite{zhang96exploiting}~derived variable elimination under CI and demonstrated the advantage of CI in inference. Besides CI, \cite{boutilier96}~proposed a decision tree model that captures context-specific independence (CSI). Later, \cite{chickering1997bayesian} extended the tree structure to decision graphs that encode equality constraints and~\cite{poole2003exploiting} derived a version of variable elimination under CSI. 
Despite showing advantages in inference, these studies---with the exception of \cite{chickering1997bayesian}---only consider the local structure of CPDs while assuming some fixed global structure; i.e., the underlying DAG for the BN is fixed and some or all of the CPTs are replaced with locally structured representations.

However, when some or all of the CPTs within some fixed global structure are replaced by locally structured representations with reduced complexity, the existing DAG structure is often not optimal or appropriate for the new representations anymore. Consider the Bayesian information criterion (BIC) that consists of the log likelihood of the data being generated by the model and a penalty for model complexity. The structured CPDs are likely to reduce the likelihood due to the so-called compression error~\citep{xiang2018compressing,zagorecki2013knowledge}, but they also have a smaller penalty as a result of using fewer parameters. These changes open up the opportunity for some alternative global structures to have better scores. Ideally, the learning algorithm should be able to choose, for example, between a CPD represented as a CPT with a smaller number of parents, and a CPD approximated as a noisy-OR with a larger number of parents. 

Assuming a fixed global structure may also lead to inaccuracies when assessing the effect of using structured representations. Compression error only measures the ability of a new representation to reproduce the original CPT, but that is not the goal of BNs. For example, the CPD modeled by noisy-OR may be different from the CPD by CPT, but with a different structure the former might be a better fit for the distribution of the data. Similarly, measuring inference error with a fixed structure is misleading. Failing to consider new structures to better accommodate alternative representations makes the false impression that we trade some posterior accuracy for reduced complexity, although in practice the posterior accuracy may even be improved with proper structure learning.

\cite{friedman1998learning}~are the first to incorporate local structures in Bayesian network structure learning (BNSL) with the \emph{score-and-search} approach. They show that using structured representations in hill-climbing allows the learning algorithm to explore more complex networks and thus avoids inferring incorrect conditional independence relations. The observation is also supported by~\cite{talvitie2019learning} in an exact search to find the optimal BN using a tree structure. Their experiments, albeit with some explicit structural constraints on the underlying DAG, suggest that structured CPDs can help the search algorithm find correct BNs with fewer samples, especially on real-world datasets. However, they also find that for some datasets CPT can still perform better. The discrepancy is likely attributed to the fact that all proposed structured representations are used separately as the sole representation for the CPDs. If the structured representations are compared with CPTs and are only used when appropriate, they can then better help the search algorithm to find the correct structure and maintain the complexity advantage in inference.

In this paper, we propose the first score-and-search approach for learning Bayesian networks with both CPT and noisy-OR relations as possible representations for CPDs. Importantly, we simultaneously learn both the global structure in the form of the underlying DAG and the local structure in the CPDs, we place no a priori constraints on the global structure, and we exactly determine all networks within a given factor of optimal. Our approach has two primary advantages. First, our approach only replaces a CPT with a noisy-OR relation when it is appropriate. Converting an arbitrary proportion of CPTs to structured representations can lead to significant degradation of the expressive power of the model, and it is difficult to determine the optimal proportion a priori. Our approach controls the degradation by specifying a Bayes factor (BF)~\citep{kass1995bayes} that measures how far a BN can deviate from the optimal network, and so only near-optimal networks with both CPTs and noisy-OR relations are learned in a principled manner. Second, our approach can scale to BNs of moderate sizes. Even local structure modelling with structured representations such as~\citep{xiang2019direct} suffers from a large search space. Our approach, on the other hand, can effectively prune most candidate parent sets of a variable by leveraging the results from learning BNs with CPTs given a BF~\citep{liao2019finding}. We empirically demonstrate that our approach can learn these \textit{mixed} BNs in a principled manner that takes advantage of a reduced complexity.

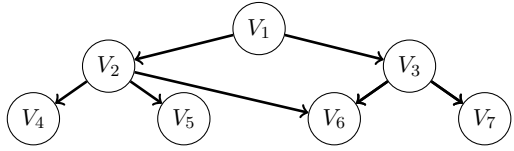
\begin{figure}[t!]
\centering
		\begin{tikzpicture}[every node/.style={circle, draw, scale=0.8, fill=gray!50}, scale=1.0, rotate = 180, xscale = -1]
		
		\node [fill=white](4) at ( 7,5.2) {$V_4$};
		\node [fill=white](5) at ( 9,5.2) {$V_5$};
		\node [fill=white](6) at ( 8,4.5) {$V_2$};
		\node [fill=white](7) at ( 11,5.2) {$V_6$};
		\node [fill=white](8) at ( 13,5.2) {$V_7$};
		\node [fill=white](9) at ( 12,4.5) {$V_3$};
		\node [fill=white](10) at ( 10,4) {$V_1$};
		
		\draw [<-, line width=1pt] (4) -- (6) ;
		\draw [<-, line width=1pt] (7) -- (6) ;
		\draw [<-, line width=1pt] (5) -- (6) ;
		\draw [<-, line width=1pt] (7) -- (9) ;
		\draw [<-, line width=1pt] (8) -- (9);
		\draw [<-, line width=1pt] (7) -- (9) ;
		\draw [<-, line width=1pt] (8) -- (9);
		\draw [<-, line width=1pt] (6) -- (10) ;
		\draw [<-, line width=1pt] (9) -- (10) ;
		
		\end{tikzpicture}
		\caption{Example Bayesian network: Each variable has the state space $\{0,1\}$. 
			 Consider the parent set of $V_6$, $\parents_6 = \{V_2, V_3\}$ The state space of $\parents_6$ is $\Omega_{\parents_6} = \{ \{0,0\}, \{0,1\},\{1,0\}, \{1,1\} \}.$ and $r_{\parents_6} = 4$. }
		\label{FIGURE:notation}
\end{figure}

%
%
\section{Background}
\label{SECTION:Background}

In this section, we review Bayesian networks~\citep{KollerF09,Darwiche09}, noisy-OR relations~\citep{good1961,Pearl1988} and the BIC scoring function~\citep{LamB94,Schwarz78}.

\subsection{Bayesian Networks}

A Bayesian network (BN) is a probabilistic graphical model that consists of a labeled directed acyclic graph (DAG), $\graph = (\vertices, \edges)$ in which the nodes $\vertices = \{V_{1}, \ldots, V_{n}\}$ correspond to random variables, the edges $\edges$ represent direct influence of one random variable on another, and each node $\vertex{i}$ is labeled with a conditional probability distribution $P(V_{i} \mid \Pi_{i})$ that specifies the dependence of the variable $V_{i}$ on its set of parents $\Pi_i$ in the DAG. A BN can alternatively be viewed as a factorized representation of the joint probability distribution over the random variables and as an encoding of the Markov condition on the nodes; i.e., given its parents, every variable is conditionally independent of its non-descendants.

In this paper, we assume that each random variable $V_i$ is binary. Each $\parents_{i}$ has state space of a set of candidate instantiations of the nodes in $\Pi_i$, $\Omega_{\parents_{i}} = \{\pi_{i1},\ldots,\pi_{i{r_{\parents_{i}}}}\}$.  We use $r_{\parents_{i}} = 2^{|\Pi_i|}$ to refer to the number of possible instantiations of the parent set $\parents_{i}$ of $V_i$ (see Figure \ref{FIGURE:notation}). 
The set $\theta = \{ \theta_{ijk} \}$ for all $ i = \{1,\ldots, n \}, j = \{1,\ldots,r_{\parents_{i}}\}$ and  $k = \{0, 1\}$ represents parameter estimates in $G$ obtained from a dataset, where each $\theta_{ijk}$ estimates the conditional probability $P(V_{i} = k \mid \Pi_{i} = \pi_{ij})$.
Given a node $V_i$ and a parent set $\Pi_i$, we define the set $\theta_i := \big\{ \theta_{ijk} \mid  j \in \{1,\ldots,r_{\parents_{i}}\}, k\in \{0,1\} \big\}$. We refer to $\theta_i$ as the \emph{full CPT} of node $i$.

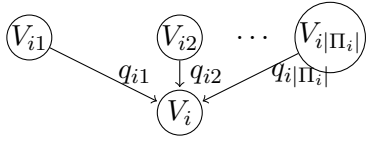
\begin{figure}[t!]
\centering
\begin{tikzpicture}[shorten >=1pt,->]
  \tikzstyle{vertex}=[circle,draw,minimum size=17pt,inner sep=0pt]
    \node[vertex] (v1) at (0,0) {$V_{i1}$};
    \node[vertex] (v2) at (2,0) {$V_{i2}$};
    \node (v3) at (3,0) {$\dots$};
    \node[vertex] (v4) at (4,0) {$V_{i|\Pi_i|}$};
    \node[vertex] (w) at (2,-1) {$V_i$};
    \draw (v2) -- node[auto,auto] (Q2) {$q_{i2}$} (w);
    \draw (v1) -- (w);
    \draw (v4) -- (w);
    \node[xshift=-0.6cm] (q1) at (Q2.west) {$q_{i1}$};
    \node[xshift=+0.9cm] (q3) at (Q2.east) {$q_{i|\Pi_i|}$};
\end{tikzpicture}
\caption{
    Causal structure for a Bayesian network with a noisy-OR 
    relation, where the set of causes ${\Pi_i} := \{V_{i1}$, \ldots, $V_{i|\Pi_i|}\}$ 
    leads to effect $V_i$ and there is a noisy-OR relation at node $V_i$.
}\label{FIGURE:causal-structure}
\end{figure}

The predominant method for Bayesian network structure learning (BNSL) from data is the \emph{score-and-search} method. Let $\graph$ be a DAG over random variables $\vertices$, and let $I = \{I_1, \ldots, I_N\}$ be a dataset, where each instance $I_i$ is an $n$-tuple that is a complete instantiation of the variables in $\vertices$. A \emph{scoring function} $\sigma( \graph \mid I )$ assigns a real value measuring the quality of $\graph$ given the data $I$. Without loss of generality, we assume that a lower score represents a better quality network structure. To simplify notation, we use $\sigma(\graph)$ in place of $\sigma(\graph \mid I)$ when the data is clear from context. 
In this paper, we focus on solving the problem of $\epsilon$-Bayesian Network Structure Learning ($\problem$) \citep{liao2019finding}.

\begin{definition}\label{DEFN:eBNSL}
	Given a non-negative constant $\epsilon$, a dataset $I = \{I_1, \ldots, I_N\}$ 
	over random variables $\vertices$ = $\{V_{1}$, \ldots, $V_{n}\}$ and a scoring 
	function $\sigma$, the $\epsilon$-Bayesian Network Structure Learning ($\problem$) 
	problem is to find all \textbf{ credible networks}, which are all networks that have a score $\score{}{\graph}$ such that $\opt \leq \score{}{\graph \mid I} \leq \opt + \epsilon$, where $\opt$ is the score of the optimal Bayesian network.
\end{definition}

It has been shown in~\citep{liao2019finding} that a good choice for $\epsilon$ is $\log BF$. By specifying the constant $\epsilon$ in terms of a Bayes factor, we can control the level of tolerance for network degradation and learn all near-optimal networks with both CPTs and noisy-OR relations as best determined by the trade-off between the fit with the data and the complexity of the model.

\subsection{BIC/MDL Scoring Function}
\label{SECTION:BIC_scoring}


In this work, we focus on the Bayesian information criterion (BIC) scoring function. As the BIC function is decomposable, when the  $\theta_i$ is given we can associate a score to a candidate parent set $\Pi_i$ of $\vertex{i}$ as follows,
\begin{equation}\label{eq:bic_local}
     BIC : \score{}{\Pi_i} =  - L(\theta_i) + t({\Pi_i})\cdot w,
\end{equation} 
where the formula consists of a term measuring the likelihood of the candidate parent set given the data and a penalty term for the number of parameters needed to specify the full CPT for the candidate parent set. Here, $L (\theta_i) =\sum_{j=1}^{r_{\Pi_i}}\sum_{k \in \{0,1\}}^{} n_{ijk} \log \theta_{ijk}$, $n_{ijk}$ is the number of instances in dataset $I$ where $V_{i}=k$ and $\Pi=\pi_{ij}$ co-occur, and $t(\Pi_i) = 2^{|\Pi_i|}$. The penalty term is weighted by $w=\log(N)/2$ where $N$ is the number of instances in dataset $I$. Note that $\score{}{\graph} = \sum_{i=1}^n \score{}{\Pi_i}$. We use the natural logarithm throughout the paper.

\subsection{Patterns for CPTs: Noisy-OR }
\label{SECTION:Patterns_for_CPTs}

With the noisy-OR relation one assumes that there are a set of causes ${\Pi_i} := \{V_{i1}$, \ldots, $V_{i|{\Pi_i}|}\}$   leading to an effect $V_i$, where $V_i, V_{ij} \in \vertices$ for all $j \in \{1,...|\Pi_i|\}$ and $V_i \notin \Pi_i$ (see Figure~\ref{FIGURE:causal-structure}). Each cause $V_{ij} \in \Pi_i$ is either present or absent, and each $V_{ij}$ in isolation is likely to cause $V_i$ and the likelihood is not diminished if more than one cause is present. Further, one assumes that all possible causes are given and when all causes are absent, the effect is absent. Finally, one assumes that the mechanism or reason that inhibits a $V_{ij}$ from causing $V_i$ is independent of the mechanism or reason that inhibits a $V_{ij'}$, $j' \not= j$, from causing $V_i$.

For a node $V_i$ and parent set $\Pi_i$, a noisy-OR relation specifies a CPT using $|\Pi_i|$ parameters, $\mathbf{q}_i = q_{i1}, \ldots, q_{i|\Pi_i|}$, one for each parent, where $q_{ij}$ is the probability that $V_i$ is false given that $V_{ij}$ is true and all of the other parents are false,
\begin{equation*}
P( V_i = 0 \mid V_{ij} = 1, V_{ij'} = 0_{[\forall j', j'\neq j]}) = q_{ij}.
\end{equation*}

\noindent
From these parameters, the full CPT representation of size
$2^{n+1}$ can be generated using,
  \begin{equation}\label{eq:noisy-OR-CPT}
    \phi_{ij0}=
    \begin{cases*}
      \text{$\prod_{j \in T_x}{ q_{ij}}$}  & \text{if $T_x \neq \{\}$} \\
         1     & \text{otherwise,}
    \end{cases*}
  \end{equation}
where $T_x = \{j \mid V_{ij} = 1\}$.
The last condition (when $T_x$ is empty) corresponds to the assumptions that all possible causes are given and that when all causes are absent, the effect is absent; i.e., $P(V_i = 0 \mid V_{i1}=0, \ldots, V_{i|{\Pi_i}|} =0) = 1$. Of course, $\phi_{ij1} = 1 - \phi_{ij0}$. The set $\phi_i := \big\{ \phi_{ijk} \mid  j \in \{1,\ldots,r_{\parents_{i}}\}, k\in \{0,1\} \big\}$ is referred to as the \emph{noisy-OR CPT} of node $i$. 

The above assumptions are not as restrictive as may first appear. One can always introduce an additional random variable $V_{i0}$ that is a parent of $V_{i}$ but itself has  no parents. The variable $V_{i0}$ represents all of the other reasons that could cause $V_{i}$ to occur. The node $V_{i0}$ and the prior probability $P(V_{i0})$ are referred to as a \emph{leak node} and the \emph{leak probability}, respectively. In this work we assume that all the causes are known. 

\section{Our Solution}
\label{SECTION:OurSolution}

In this section, we  present our \emph{score-and-search} approach for learning all Bayesian networks, given local scores, that are within a given factor $\epsilon$ of optimal, where the networks can contain both full CPT and noisy-OR relations as possible representations for the CPDs. In general, a score-and-search approach scores candidate parent sets for the nodes in the network and searches for the choice of a parent set, one for each node, that leads to the best overall score while ensuring that the network is acyclic. Before presenting our overall approach for solving $\problem$ (Section~\ref{sec:overall-algorithm}), we first describe an effective gradient descent algorithm to score a candidate noisy-OR relation using the widely used BIC score (Section~\ref{sec:noisy-OR-scoring}) and pruning rules that allow the search to scale to larger networks (Section~\ref{sec:pruning}).

\subsection{BIC Score for Noisy-OR Relations}\label{sec:noisy-OR-scoring}

The BIC score consists of a maximum likelihood term and a penalty term. We present a gradient descent algorithm that is based on minimizing a KL divergence as it is known that minimizing the KL divergence results in maximizing the likelihood (see, e.g., \cite{murphy2012}). Recall that the elements of $\theta_i$ are conditional probabilities computed from the dataset $I$. Given a node $V_i$, we must compute maximum likelihood estimates for the noisy-OR CPT $\mathbf{\phi}_i$ for every candidate parent set $\Pi_i$, such that the conditional KL divergence between the full CPT $\theta_i$ and the resulting noisy-OR CPT $\phi_i$ that is determined by the $\mathbf{q}_i$ (see Equation~\ref{eq:noisy-OR-CPT}), is minimized. 
Note that the KL divergence between two \emph{conditional} probability distributions, $P(A|B)$ and $Q(A|B)$ is given by,
\begin{equation}\nonumber
D_{\mathit{KL}} (P(A|B) \mid \mid Q(A|B))
	 = \sum_{b \in B}P(B=b) \sum_{a \in A} P(A=a|B=b)\log\frac{ P(A=a|B=b)}{ Q(A=a|B=b)} .
\end{equation}
We note that an alternative approach to estimate noisy-OR parameters is to maximize the log-likelihood using the expectation-maximization (EM) technique, which was derived in \cite{Dempster1977} and applied to noisy-OR in \cite{vomlel2006}. We perform an experimental comparison of the two approaches in Section~\ref{SECTION:Experimental}.


To derive our gradient descent algorithm, we begin with the definition of KL divergence for the two conditional probability distributions, $\theta_i$ and $\phi_i$, and rewrite it into a more convenient form:
\begin{align*}
	\ D_{\mathit{KL}} (\theta_i \mid \mid \phi_i)
	&\stackbin[]{0}{=} \sum_{j=1}^{r_{\Pi_i}}P(\pi_{ij}) \sum_{k \in \{0,1\}} \theta_{ijk}\log\frac{ \theta_{ijk}}{ \phi_{ijk}}\\
	&\stackbin[]{1}{=}\sum_{j=1}^{r_{\Pi_i}}\frac{n_{ij}}{N} \sum_{k \in \{0,1\}} \theta_{ijk}\log\frac{ \theta_{ijk}}{ \phi_{ijk}}\\
	&\stackbin[]{2}{=}\frac{1}{N} \sum_{j=1}^{r_{\Pi_i}}\sum_{k \in \{0,1\}} n_{ijk} \cdot \theta_{ijk}\log\frac{ \theta_{ijk}}{ \phi_{ijk}}\\
	&\stackbin[]{3}{=}\frac{1}{N} \sum_{j=1}^{r_{\Pi_i}}\sum_{k \in \{0,1\}} n_{ijk} \cdot \theta_{ijk}\log{ \theta_{ijk}} 
	-\frac{1}{N} \sum_{j=1}^{r_{\Pi_i}}\sum_{k \in \{0,1\}} n_{ijk} \cdot \theta_{ijk}\log{ \phi_{ijk}} ,
\end{align*}
where $N$ is the number of instances in our dataset. To find $\phi_i$ such that $D_{\mathit{KL}} (\theta_i \mid \mid \phi_i)$ is minimized, note that the first term in Step 3 is constant. So, we must determine,
\begin{align*}
    \underset{\mathbf{q_i}}{\operatorname{argmin}}\ D_{\mathit{KL}} (\theta_i \mid \mid \phi_i) 
    &= - \sum_{j=1}^{r_{\Pi_i}}\sum_{k \in \{0,1\}} n_{ijk} \cdot \theta_{ijk}\log{ \phi_{ijk}} ,
\end{align*}
where the $\mathbf{q}_i$ that minimizes the KL divergence are the maximum likelihood estimates for $\phi_i$ that are determined by the $\mathbf{q}_i$ (Equation~\ref{eq:noisy-OR-CPT}).
The penalty term in the BIC score can be computed in constant time; specifically, the number of parents in the candidate parent set. Thus, fitting these noisy-OR parameters gives us the BIC score for the noisy-OR for a candidate parent set. To find these noisy-OR parameters, we use Algorithm 1 which performs gradient descent for the derivative,
\begin{equation}\label{eq:gradient}
    \Delta_{\mathit{KL}}^{\mathbf{q}_i} = \frac{d}{d\mathbf{q}_i} \sum_{j=1}^{r_{\Pi_i}}\sum_{k \in \{0,1\}} n_{ijk} \cdot \log{ \phi_{ijk}} .
\end{equation}
We start with an initial guess for the set of noisy-OR parameters $\mathbf{q}_i$ and evaluate term $\Delta_{\mathit{KL}}^{\mathbf{q}_i}$ for these values (Equation~\ref{eq:gradient}). The initial guess uses hot starts in that the solution for a smaller candidate parent set is used as the starting point when estimating the parameters for a candidate set that is a superset. We perform gradient descent over $\mathbf{q}_i$, where each step update is found by a simple geometric line search algorithm (see Algorithm~\ref{alg:algorithm}). Geometric line search is a backtracking line search procedure, where we first choose a descent direction and then determine the maximum amount to move along that direction.

\begin{algorithm}[tb]
\caption{Computing Noisy-OR Parameters for a Candidate Parent Set}
\label{alg:algorithm}
\textbf{Input}: Node $V_i$, candidate set $\Pi_i$, a dataset $I$ of $N$ instances.\\
\textbf{Parameter}: Threshold $t$, maximum iterations $\mathit{maxIter}$\\
\textbf{Output}: A set of noisy-OR parameters : $\mathbf{q}_i = q_{i1}, \ldots, q_{i|\Pi_i|}$
\begin{algorithmic}[1] 
\STATE Initialize $\mathbf{q}_i = q_{i1}, \ldots, q_{i|\Pi_i|}=hotstarts()$
\STATE Initialize $l=0, \mathbf{mq}_i = \mathbf{q}_i, \delta=\infty$
\WHILE{$l < \mathit{maxIter}$}
\STATE $\mathbf{q}^{\prime}_i = \mathbf{q}_i$
\STATE $step = \mathit{GeometricLineSearch}(\mathbf{q}^{\prime}_i , \Delta_{\mathit{KL}}^{\mathbf{q}^{\prime}_i})$
\STATE $\mathbf{q}_i = \mathbf{q}^{\prime}_i  - step* \Delta_{\mathit{KL}}^{\mathbf{q}^{\prime}_i}$
\STATE $\delta{q}_i =\Delta_{\mathit{KL}}^{\mathbf{q}_i} - \Delta_{\mathit{KL}}^{\mathbf{q}^{\prime}_i}  $
\IF {$\delta{q}_i < \delta$}
\STATE $ \mathbf{mq}_i = \mathbf{q}_i$
\STATE $\delta = \delta{q}_i$
\ENDIF
\IF {$\delta{q}_i < t$} 
\STATE \textbf{break}
\ENDIF
\STATE $l = l + 1$
\ENDWHILE
\STATE \textbf{return} $\mathbf{mq}_i$
\end{algorithmic}
\end{algorithm}



\subsection{Pruning Rules} \label{sec:pruning}

To find all near-optimal BNs given an approximating factor $\epsilon$ for a dataset $I$, we propose to compute two different sets of local scores for each node. The first set is the BIC scores when the conditional probability distributions for the candidate parents sets are represented by full CPTs. The second set is the BIC scores when the conditional probability distributions for the candidate parent sets are represented  by noisy-OR relations. However, computing the local scores for all nodes is quite cost prohibitive---we would need a set of $n\cdot2^{n-1}$ local scores for each of the two BIC scores. A solution is to prune the search space of candidate parent sets, provided that global optimality constraints of the full network structure are not violated. 
Adopting the terminology of \cite{liao2019finding}, we say that a candidate parent set $\Pi_i$ can be \textit{safely pruned} given a non-negative constant $\epsilon \in \mathbb{R}^+$ if $\Pi_i$ cannot be the parent set of $V_i$ in any network in the set of credible networks (see Definition~\ref{DEFN:eBNSL}).
For computing BIC scores for full CPTs, we employ the following two pruning rules given by \cite{liao2019finding} to find all near-optimal Bayesian networks given an approximating factor $\epsilon$. 

\begin{lemma}
	Given a node $\vertex{j}$, candidate parent sets
	$\Pi_j$ and  $\Pi_j^{\prime}$, and some $\epsilon\in \mathbb{R}^+$,  if $\Pi_j \subset \Pi_j^{\prime}$ and $\score{}{\Pi_j} + \epsilon \leq \score{}{\Pi_j'}$,
	$\Pi_j^{\prime}$ can be safely pruned. \label{lem:scoreprune}
\end{lemma}
\begin{theorem}
    Given a node $\vertex{j}$,  candidate parent sets
	$\Pi_j$ and $\Pi_j'$, and some $\epsilon \in \mathbb{R}^+$,
	if $\Pi_j \subset \Pi_j'$ and $\score{}{\Pi_j} -  \pen{}{\Pi_j'} + \epsilon < 0$,
	$\Pi_j'$ and all supersets of $\Pi_j^{\prime}$ can be safely pruned  if $\sigma$ is the BIC scoring function.
	\label{thm:scoreprune2}
\end{theorem}

\noindent
For computing BIC scores for noisy-OR relations, we introduce two new pruning rules. 

\begin{lemma}
A candidate parent set $\Pi_i$ of a node $V_i$ that is consistently instantiated to zero throughout the dataset whenever the node is set to one can be safely pruned. 
\end{lemma}
\begin{proof}
The candidate parent set $\Pi_i$ cannot explain $V_i$ in this configuration as there is no instance in the data file to indicate that $\Pi_i$ affects the values of $V_i$.
\end{proof}

\begin{theorem}
 Given a node $\vertex{j}$ and some $\epsilon\in \mathbb{R}^+$, a candidate parent set $\Pi_i$ with its penalty term greater than the sum of the score of the null parent set and $\epsilon$ can be safely pruned.
\end{theorem}
\begin{proof}
The null set is a subset of all candidate parent sets and by Lemma \ref{lem:scoreprune} any candidate parent set with a score exceeding the score of the null parent set can be safely pruned. 
Consider the definition of BIC for a parent set $\Pi_i$ for node $V_i$,
  $  \score{}{\Pi_i} = - L(\theta_i) + t({\Pi_i})\cdot w$.
Let us have a candidate parent set with 2 or more parents, and with its penalty term greater than the score of the null parent set for $V_i$, $\score{i}{\{\}}$. Such a parent set will score lower than the null parent set as log-likelihood is negative and can be safely pruned; i.e., $t({\Pi_i})\cdot w > \score{i}{\{\}} + \epsilon \Rightarrow -L(\theta_i)  + t({\Pi_i})\cdot w  >  \score{i}{\{\}}  + \epsilon\Rightarrow  \score{}{\Pi_i}  >  \score{i}{\{\}}+ \epsilon$.
\end{proof}

\subsection{\texorpdfstring{Algorithm for $\problem$}{Algorithm}}\label{sec:overall-algorithm}

Here we give our overall algorithm for $\problem$, a principled way to automatically select between full CPTs and noisy-OR relations, given a dataset and an approximation factor $\epsilon$.

\begin{itemize}
\item%
\textbf{Step 1.} Determine the BIC scores when fitting a full CPT for all candidate parent sets that could not be pruned with pruning rules from \citep{liao2019finding} using Equation \ref{eq:bic_local}.
\item%
\textbf{Step 2.} Determine the BIC scores when fitting a noisy-OR relation for all candidate parent sets that could not be pruned using our pruning rules in Section \ref{sec:pruning}. Here, the noisy-OR parameters are fit using Algorithm \ref{alg:algorithm}, which minimizes the KL divergence between the full-CPT and the noisy-OR CPT. These parameters are used to compute the noisy-OR BIC score. 
\item%
\textbf{Step 3.} Merge these two score sets, using pruning rules Lemma \ref{lem:scoreprune} and Theorem \ref{thm:scoreprune2}, into a list of scores for candidate parent sets for each node in the dataset. During merging scores of a node $V_j$, we have to examine only cases where a candidate parent set $\Pi_j$ belongs to the set of BIC scores and its superset $\Pi_j'$ belongs to the noisy-OR BIC scores and vice versa.%
\item
\textbf{Step 4.} The scores obtained in Step 3 are used to learn the set of credible networks using a developmental version of GOBNILP \citep{cussens2015gobnilp}, $gobnilp{\_dev}$ \citep{liao2019finding}, which can be used to solve the $\problem$ problem and collect all the networks in the credible set for a given approximation factor. 
\end{itemize}
%

\section{Experimental Evaluation}
\label{SECTION:Experimental}

In this section, we show the accuracy of Algorithm \ref{alg:algorithm}\footnote{Code available at \url{https://github.com/CharupriyaSharma/eBNSLNoisyOR}}  in computing the noisy-OR parameters for synthetic BNs with embedded noisy-OR relations. We also show significant presence of noisy-OR relations in standard benchmark networks. Finally, we test the performance of our learned networks against ground truth networks. 
All experiments are conducted on computers with 2.2 GHz Intel E7-4850V3 CPUs. Each experiment is limited to 64 GB of memory and 24 hours of CPU time.

\begin{table}[t]
\small
\begin{tabular}{cc}
\begin{tabular}{@{}crrrrrr@{}}
\toprule
 Parent &  \multicolumn{2}{c}{$N=100$} &  
    \multicolumn{2}{c}{$N=500$} &
    \multicolumn{2}{c@{}}{$N=1000$} \\
 Size & 
    \multicolumn{1}{c}{KL} & \multicolumn{1}{c}{EM} & 
    \multicolumn{1}{c}{KL} & \multicolumn{1}{c}{EM} & 
    \multicolumn{1}{c}{KL} & \multicolumn{1}{c@{}}{EM} \\
 \midrule
2 & 0.16 & 1.07 & 0.07 & 1.07 & 0.05 & 1.00 \\
3 & 0.21 & 1.18 & 0.09 & 1.11 & 0.07 & 1.07 \\
4 & 0.27 & 1.04 & 0.11 & 1.27 & 0.07 & 1.26 \\
5 & 0.25 & 1.50 & 0.11 & 1.54 & 0.08 & 1.57 \\
6 & 0.34 & 1.99 & 0.16 & 2.04 & 0.10 & 2.06 \\
7 & 0.41 & 2.09 & 0.24 & 2.03 & 0.16 & 1.99 \\
\bottomrule
\end{tabular}
&
\begin{tabular}{@{}crrrrrr@{}}
\toprule
 Parent &  \multicolumn{2}{c}{$N=100$} &  
    \multicolumn{2}{c}{$N=500$} &
    \multicolumn{2}{c@{}}{$N=1000$} \\
 Size & 
    \multicolumn{1}{c}{KL} & \multicolumn{1}{c}{EM} & 
    \multicolumn{1}{c}{KL} & \multicolumn{1}{c}{EM} & 
    \multicolumn{1}{c}{KL} & \multicolumn{1}{c@{}}{EM} \\
 \midrule
2 & 0.04 & 0.05 & 0.01 & 0.01 & 0.00 & 0.01 \\
3 & 0.13 & 0.32 & 0.02 & 0.05 & 0.01 & 0.03 \\
4 & 0.33 & 1.32 & 0.06 & 0.24 & 0.03 & 0.12 \\
5 & 1.02 & 4.91 & 0.18 & 1.20 & 0.07 & 0.55 \\
6 & 1.33 & 9.02 & 0.33 & 3.48 & 0.16 & 2.06 \\
7 & 2.54 & 12.08 & 1.07 & 11.68 & 0.60 & 6.55 \\
\bottomrule
\end{tabular}
\end{tabular}
\caption{(Left) Median relative error in noisy-OR parameters and (right) median conditional KL divergence of noisy-OR CPTs learned by Algorithm~\ref{alg:algorithm}, denoted KL, and the expectation-maximization algorithm, denoted EM, from ground truth for various parent set sizes.}
\label{TABLE:RelativeErrorAndKL}
\end{table}

\subsection{Recovery of Noisy-ORs in Synthetic Datasets}

To evaluate the accuracy of Algorithm~\ref{alg:algorithm} in finding the noisy-OR parameters and minimizing conditional KL divergence, we used synthetic BNs which consisted of a single noisy-OR. The parent set sizes were in the range \{2, \ldots, 7\}, all parent nodes had priors of 0.5, and the noisy-OR parameters $\mathbf{q} = q_{1}, \ldots, q_{|\Pi|}$ in the ground truth were uniformly sampled from the set \{0.01, 0.02, \ldots, 0.99\}. Thirty tests were performed at each parent set size.

We randomly generated datasets from the synthetic BNs with 100, 500 and 1000 instances, respectively. Algorithm~\ref{alg:algorithm} was applied to a dataset and the noisy-OR parameters estimated by the algorithm were compared against the parameters in the ground truth network (see Table \ref{TABLE:RelativeErrorAndKL}). As well, the conditional KL divergence was computed between the noisy-OR CPT for the estimated parameters and the noisy-OR CPT for the ground truth parameters (see Equation~\ref{eq:noisy-OR-CPT}). We also compared our results against the expectation-maximization algorithm for noisy-OR proposed by \cite{vomlel2006}, the code for which was supplied by the author. 
As shown in Table \ref{TABLE:RelativeErrorAndKL}, Algorithm \ref{alg:algorithm} estimated the ground truth parameters with significantly higher accuracy than the EM algorithm. Algorithm \ref{alg:algorithm} also had much lower conditional KL divergence. 


\subsection{Experiments on Standard Benchmarks: Presence of Noisy-OR Relations}

To evaluate the ability of our overall algorithm for $\problem$ (see Section~\ref{sec:overall-algorithm}) to learn networks with noisy-OR relations, we used standard datasets from the UCI Machine Learning Repository (\url{https://archive.ics.uci.edu/}). The datasets used were all binary or made binary. 

\begin{table}[htb]
\small
\centering
\begin{tabular}{cc}
\begin{tabular}{lrrrrr}\toprule
 Dataset & $n$ & $N$ & nodes  &	 ave. & max.   \\\midrule
 adult      & 14 & 32,561 &     0 &   0.0  &   0.0\\
 nltcs      & 16 &  3,236 &     0 &   0.0  &   0.0\\
 msnbc      & 17 & 58,265 &     0 &   0.0  &   0.0\\
 zoo        & 17 &    101 &     7 &  41.9  &  99.4\\
 letter     & 17 & 20,000 &    OT &    OT  &    OT\\  
 hepatitis  & 20 &    155 &    10 &  76.9  & 100.0\\
 parkinsons & 23 &    195 &    11 &  51.2  & 100.0\\
 sensors    & 25 &  5,456 &    OT &    OT  &    OT\\
 \bottomrule
\end{tabular}
&
\begin{tabular}{lrrrrr}\toprule
 Dataset & $n$ & $N$ & nodes  &	 ave. & max.   \\\midrule
 autos      & 26 &    159 &    13 &  76.0  & 100.0\\
 horse      & 28 &    300 &     3 &  97.4  & 100.0\\
 flag       & 29 &    194 &    10 &  77.6  & 100.0\\ 
 wdbc       & 31 &    569 &    OT &    OT  &    OT\\
 soybean    & 36 &    266 &     9 &  86.1  & 100.0\\
 alarm      & 37 &  1,000 &     2 &  28.8  &  56.4\\
 bands      & 37 &    277 &     8 &  63.4  & 100.0\\
 spectf     & 45 &    267 &     0 &   0.0  &   0.0\\
 \bottomrule
\end{tabular}
\end{tabular}
\caption{Total number of nodes where a noisy-OR relation is selected (nodes) and average (ave.) and maximum (max.) percentage of networks in the set of credible networks that select noisy-OR relations for these nodes, for various benchmarks with $n$ nodes and $N$ instances in the dataset. OT indicates a dataset that could not be solved within the time limit.}
\label{TABLE:standard-benchmarks}
\end{table}

The overall algorithm for $\problem$ was applied to a dataset to learn the set of credible networks using a Bayes Factor, $BF=20$. Out of the 13 (in a total of 16) benchmarks the algorithm was able to solve, 9 benchmarks showed a presence of noisy-OR relations (see Table~\ref{TABLE:standard-benchmarks}). Specifically, these 9 benchmarks had 2 or more nodes that were assigned noisy-OR relations in at least 28\% of the networks in the credible set. Also, 7 benchmarks had at least one node that was assigned a noisy-OR relation in all of the networks in the credible set. Note that some benchmarks, such as hepatitis and parkinsons, select noisy-OR relations for around half of their nodes, which shows that using only full CPTs could have resulted in overfitting. Further, optimal BNs containing noisy-OR relations were consistently found to have better scores than that of optimal networks found using only full CPTs. We also examined the effectiveness of the pruning rules (Steps 2 and 3 of the algorithm). On these benchmarks, the rules safely pruned away from 89.17\% to 99.99\% of the candidate parent sets, showing that the pruning rules are highly effective.

\subsection{Performance on Ground Truth Networks}

To further evaluate our overall algorithm for $\problem$ (see Section~\ref{sec:overall-algorithm}), we used real-world Bayesian networks from the Bayesian Network Repository (\url{www.bnlearn.com/bnrepository}). The variables in the networks were made binary and their corresponding CPTs compressed (see Table~\ref{TABLE:inference-error}; BNs without a \_b suffix were already binary). From each ground truth network, we randomly generated datasets with 100, 500, and 1000 samples. We then ran our structure learning algorithm on the datasets to learn the set of credible networks, fixed the CPT parameters using maximum likelihood estimation and measured relative inference error against the ground truth network.

Table \ref{TABLE:inference-error} shows the median relative inference error of the best scoring and the worst scoring networks in the set of credible networks, as well as that of the best-scoring network with full CPTs (i.e., not containing noisy-OR relations), against that of the ground truth network. Overall the inference error of the best scoring network is comparable to that of the full CPT. Somewhat surprisingly, the error for the worst scoring network can be smaller than for the best scoring network or the full CPT. 

\begin{table}[thb]
\small
\centering
\begin{tabular}{lrrrrrrrrrr}
\toprule
Bayesian & & \multicolumn{3}{c}{$N=100$} & \multicolumn{3}{c}{$N=500$} & \multicolumn{3}{c}{$N=1,000$}\\
network & $n$  & best & worst &CPT  &best  &worst &CPT &best & worst &CPT \\
\midrule
earthquake   &  5 & 0.03 & 0.91 &0.00  & 0.02 & 0.98 & 0.00 &0.26 & 0.53& 1.00\\
survey\_b    &  6 & 0.05 & 0.69 &0.00  &0.02& 0.74 &0.00  &0.01 & 0.75&0.00 \\
asia         &  8 & 0.04 & 0.13 &0.92  &0.04 & 0.92 &0.02 &0.02 & 0.90& 0.08\\
sachs\_b     & 11 & 0.43 & 0.68 & 0.18 &0.70 & 0.60 &0.21 &0.68 & 0.62&0.01\\
child\_b     & 20 & 0.05 & 0.91 &0.01 &0.05 & 0.88 &0.07 &0.05 & 0.85& 0.04\\
insurance\_b & 27 & 0.67 & 0.72 & 0.70 & 0.65 &0.71 &0.68 &0.65 &0.68& 0.68\\
alarm\_b     & 37 & 0.04 & 0.99 &0.01 &0.08 & 0.99 &0.05 &0.05 & OT &0.06\\
\bottomrule 
\end{tabular}
\caption{Median relative inference error for the best and worst scoring network in the set of credible networks learned by Algorithm \ref{alg:algorithm} and the full CPT against the ground truth network. The datasets with $N$ instances were generated from various ground truth BNs with $n$ nodes.}
\label{TABLE:inference-error}
\end{table}

To perform inference on our learned set of credible networks, we generated evidence for $10\%$ of nodes in the network. The nodes were randomly selected. For one trial, we selected a state of every node in the evidence, which was set according to the node’s posterior probability distribution in the model, conditional on the evidence observed up till this point. Then, we computed the posterior probability distributions over the non-evidence nodes for our learned network and for the ground truth network. The inference errors were the differences between these values. We repeated the described procedure 1000 times for each of the networks. Inference was performed using JavaBayes (\url{www.cs.cmu.edu/~javabayes}), which was extended to take in an evidence file and two BNs for comparison. 
Our results our consistent with \cite{zagorecki2013knowledge}, who show that in three real-world Bayesian networks, noisy-OR/MAX relations were a good fit for up to 50\% of the CPTs in these networks and that converting some CPTs to noisy-OR/MAX relations gave good approximations when answering probabilistic queries.



\section{Conclusion}

Existing successful approaches for learning Bayesian networks from data use the well-known score-and-search approach. We extend the score-and-search approach to simultaneously learn the best global structure and the best local structure when the choice is either a full CPT or a noisy-OR relation for a candidate parent set of a node in the network. We show how to score a causal noisy-OR relation for a candidate parent set by fitting the best possible noisy-OR to the data, and we show how to effectively prune the search space while maintaining the optimality of the networks that are learned. Our experimental results provide evidence of the effectiveness of our approach. In particular, it was found that noisy-OR relations appeared in a significant proportion of the learned networks, for well known datasets. 

\bibliography{probabilistic,scip}

\end{document}